%% file: main.tex
\documentclass[10pt]{article} 

\usepackage[accepted]{imports/rlj} 

\usepackage{amssymb}            
\usepackage{mathtools}          
\usepackage{mathrsfs}           
\usepackage{graphicx}           
\usepackage{subcaption}         
\usepackage[space]{grffile}     
\usepackage{url}                
\usepackage{lipsum}             

\title{An Analysis of Action-Value Temporal-Difference\\ Methods That Learn State Values}

\setrunningtitle{An Analysis of Action-Value TD Methods That Learn State Values}

\author{
Brett Daley\textsuperscript{1,2,$\dagger$},
Prabhat Nagarajan\textsuperscript{1,2,$\dagger$},
Martha White\textsuperscript{1,2,3},\\
Marlos C.\ Machado\textsuperscript{1,2,3}
}

\emails{\{brett.daley,nagarajan,whitem,machado\}@ualberta.ca}

\affiliations{
$^{1}$\textbf{Department of Computing Science, University of Alberta, Canada}\\
$^{2}$\textbf{Alberta Machine Intelligence Institute (Amii)}\\
$^{3}$\textbf{Canada CIFAR AI Chair}
\par 
$^\dagger$ Co-first authors.
}

\contribution{
    We prove the expected contraction of QV-learning for on-policy prediction.
    }
    {
    \citet{wiering2005qv} introduced the QV-learning algorithm, but omitted a convergence proof.
    To our knowledge, there is no published convergence proof of QV-learning to date.
    }

\contribution{
We raise the issue that QVMAX, the main off-policy control variant of QV-learning, is biased.
We empirically demonstrate that such bias can significantly impact performance.
We then introduce a new, unbiased algorithm, BC-QVMAX, and empirically demonstrate that it converges to a similar solution to that of Q-learning in our considered settings.
}
{\citet{wiering2009qv} introduced QVMAX without theoretical justification, heuristically mirroring the Q-learning update.
Its bias has not been identified until~now.
Our empirical results were obtained with parametric MDP experiments that we designed;
they should be interpreted only as anecdotal evidence for the convergence of BC-QVMAX.
}

\contribution{In the context of AV-learning algorithms, we formalize a tabular version of Dueling DQN, and we introduce a new algorithm, Regularized Dueling Q-learning (RDQ). RDQ addresses the identifiability issue of the naive dueling decomposition by using an $l_2$ penalty instead of subtracting the mean advantage.
We empirically demonstrate that, given the same network architecture, RDQ significantly outperforms Dueling DQN in the MinAtar domain.
}
    {
    \citet{wang2016dueling} introduced Dueling DQN from the perspective of an improvement to the neural network architecture used by DQN.
    RDQ is the result of relaxing the semantics behind estimating $Q(s,a)$ through two value functions. Instead of generating $Q(s,a)$ from approximations of $V(s)$ and $A(s,a)$, RDQ searches for the closest point on the hyperplane defined by two arbitrary functions that sum to $Q(s,a)$.
    We used tuned versions of DQN and Dueling DQN as baselines for MinAtar~\citep{ceron2021revisiting}.
    }


\summary{
\input{paper_abstract.tex}
}

\graphicspath{{images/}{images/minatar/}{images/tabular/}}
\input{imports/math_commands.tex}

\usepackage{paper_preamble}

\begin{document}

\makeCover  
\maketitle  

\input{paper_body.tex}


\appendix

\subsubsection*{Acknowledgments}
\label{sec:ack}

This research was supported in part by the Natural Sciences and Engineering Research Council of Canada (NSERC) and the Canada CIFAR AI Chair Program.
Prabhat Nagarajan has been additionally supported by the Alberta Innovates Graduate Student Scholarship.
Computational resources were provided in part by the Digital Research Alliance of Canada.


\bibliography{main}
\bibliographystyle{imports/rlj}

\beginSupplementaryMaterials

\input{paper_appendix.tex}


\end{document}

%% file: paper_abstract.tex
The hallmark feature of temporal-difference (TD) learning is bootstrapping:
using value predictions to generate new value predictions.
The vast majority of TD methods for control learn a policy by bootstrapping from a single action-value function (e.g., Q-learning and Sarsa).
Significantly less attention has been given to methods that bootstrap from two asymmetric value functions:
i.e., methods that learn state values as an intermediate step in learning action values.
Existing algorithms in this vein can be categorized as either QV-learning or AV-learning.
Though these algorithms have been investigated to some degree in prior work, it remains unclear if and when it is advantageous to learn two value functions instead of just one---and whether such approaches are theoretically sound in general.
In this paper, we analyze these algorithmic families in terms of convergence and sample efficiency.
We find that while both families are more efficient than Expected Sarsa in the prediction setting, only AV-learning methods offer any major benefit over Q-learning in the control setting.
Finally, we introduce a new AV-learning algorithm called Regularized Dueling Q-learning (RDQ), which significantly outperforms Dueling DQN in the MinAtar benchmark.

%% file: imports/math_commands.tex
\usepackage{amsmath,amsfonts,bm}

\def\eqref#1{equation~\ref{#1}}

\def\1{\bm{1}}

\def\vzero{{\bm{0}}}
\def\vone{{\bm{1}}}

\def\vtheta{{\bm{\theta}}}

\def\vb{{\bm{b}}}

\def\vq{{\bm{q}}}
\def\vr{{\bm{r}}}

\def\vv{{\bm{v}}}

\def\vy{{\bm{y}}}

\def\mA{{\bm{A}}}

\def\mE{{\bm{E}}}

\def\mP{{\bm{P}}}

\DeclareMathAlphabet{\mathsfit}{\encodingdefault}{\sfdefault}{m}{sl}
\SetMathAlphabet{\mathsfit}{bold}{\encodingdefault}{\sfdefault}{bx}{n}

\newcommand{\E}{\mathbb{E}}

\newcommand{\R}{\mathbb{R}}



%% file: paper_body.tex
\begin{abstract}
    \input{paper_abstract.tex}
\end{abstract}

\section{Introduction}

Reinforcement learning (RL) is the study of decision-making agents that interact with their environment to maximize a notion of cumulative reward.
Temporal-difference (TD) learning \citep{sutton1988learning} is among the most widely used classes of RL algorithms.
Like dynamic programming \citep{bellman1957dynamic}, TD methods learn one or more value functions to guide policy improvement.
However, unlike dynamic programming, TD approaches do not assume access to or knowledge of the environment's dynamics.
As such, value-function estimates must be generated iteratively through direct interaction with the MDP.

A key feature of TD is bootstrapping:
constructing new value predictions from other value predictions.
How such predictions should be integrated for the purposes of bootstrapping to learn as efficiently as possible has been the subject of much RL research.
Classic methods like Q-learning \citep{watkins1989learning}, Sarsa \citep{rummery1994line}, and Expected Sarsa \citep{john1994when} all learn action-value functions, and differ only by their bootstrapped targets, but exhibit dramatically different properties in terms of risk aversion, bias-variance trade-off, and convergence.

The vast majority of TD algorithms that attempt to learn good control policies---including the three algorithms mentioned above---rely on just a single action-value function, denoted by $Q(s,a)$ where $(s,a)$ is a state-action pair.
A far less common paradigm is to jointly learn a secondary state-value function, $V(s)$, in the process of learning $Q(s,a)$.
That is to say, such TD methods learn state values as an intermediate step in learning action values.\footnote{
    Note that this precise definition of learning state and action values excludes double action-value methods such as Double Q-learning \citep{van2010double}, which are beyond the scope of this paper.
}

We broadly categorize such methods as either \emph{QV-learning} or \emph{AV-learning}.\footnote{
    We introduce the term \emph{AV-learning} to refer to value-based RL methods that learn both advantage and state-value functions.
    AV-learning should not be confused with VA-learning \citep{tang2023va}, a specific algorithmic instance of AV-learning.
    Although VA-learning is inspired by Dueling DQN, it is not the exact tabular analog of it, which we derive in \Cref{subsec:dueling_q-learning}.
}
In QV-learning, the agent directly estimates $Q(s,a)$ and $V(s)$, with some bootstrapping between the two value functions.
For instance, the classic QV-learning method \citep{wiering2005qv} essentially replaces the bootstrap term of Expected Sarsa with a learned approximation generated by TD(0) \citep{sutton1988learning}.
Conversely, AV-learning methods make use of the advantage decomposition \citep{baird1993advantage}, which breaks down the action value as $Q(s,a) = V(s) + \Adv(s,a)$.
In contrast to QV-learning, the two value functions do not directly bootstrap from one another, but rather bootstrap from the composite action-value function, $Q(s,a)$.
AV-learning techniques were popularized in deep RL by the dueling network architecture \citep{wang2016dueling}.
Analyzing the behavior of its underlying algorithm, Dueling Q-learning, is important to understand how these networks learn.

In spite of past empirical evidence supporting QV-learning \citep[e.g.,][]{sabatelli2020deep,modayil2023towards} and AV-learning \citep{baird1993advantage,wang2016dueling,tang2023va}, the exact circumstances and mechanisms behind these improvements remain unclear.
After all, estimating two value functions instead of one inherently requires more parameters to be learned.
On the other hand, it seems that there are opportunities for synergistic information sharing between the value functions which could potentially explain the observed sample-efficiency gains.
Still, it is unknown when, or how reliably, such gains can be obtained.
Furthermore, several of these methods---particularly in the QV-learning family---are missing formal convergence guarantees.

Toward an understanding of when and how state-value estimation can be leveraged to learn action values more efficiently (and soundly), we investigate the theoretical underpinnings of both QV-learning and AV-learning algorithms and conduct experiments to isolate some of their unique properties.
We show that QV-learning can be more efficient than Expected Sarsa for on-policy prediction and that the performance gap increases with the cardinality of the action set.
However, QV-learning's main extension for off-policy control, QVMAX \citep{wiering2009qv}, suffers from bias;
even when we correct this bias, we show empirically that its sample efficiency is surpassed by that of Q-learning.
In contrast, we find that AV-learning methods work much more consistently for control problems overall and easily outperform Q-learning.

We also discuss the implicit assumptions underlying Dueling Q-learning and introduce an algorithm called Regularized Dueling Q-learning (RDQ) that converges to a different pair of value functions.
An advantage of RDQ is that it easily extends to function approximation, and we show that it is significantly more sample efficient than Dueling DQN \citep{wang2016dueling} in the five MinAtar games \citep{young2019minatar} when using the same network architecture and hyperparameters.
All of our experiment code is available online.\footnote{
    \url{https://github.com/brett-daley/reg-duel-q}
}
Our results help to characterize the nuanced distinction between these two algorithm families, and ultimately motivate state-value learning as a sound and effective way to accelerate action-value learning.

\section{Background}
\label{sec:background}

We formalize the RL problem as a finite Markov decision process (MDP) described by $(\S, \A, p, \mathcal{R}, \gamma)$.
At each time step $t \geq 0$, the agent observes the environment state, $S_t \in \mathcal{S}$, and executes an action, $A_t \in \A$.
The environment consequently transitions to a new state, $S_{t+1} \in \S$, and returns a reward, $R_{t+1} \in \mathcal{R}$, with probability $p(S_{t+1},R_{t+1} \mid S_t, A_t)$.

In the policy-evaluation setting, the agent's goal is to learn the action-value function:
the expected discounted return achieved by each state-action pair $(s,a)$ under a fixed agent behavior.
The agent's behavior is defined by a policy, a mapping from states to distributions over actions.
Letting $G_t = \sum_{i=0}^\infty \gamma^i R_{t+1+i}$ be the discounted return at time $t$, the action-value function for policy $\pi$ is defined as
\looseness=-1
\begin{equation*}
    q_\pi(s,a) = \mathbb{E}_\pi[G_t \mid (S_t,A_t) = (s,a)]
    \,,
\end{equation*}
where the expectation $\mathbb{E}_\pi$ indicates that actions are sampled according to $\pi$.
For any policy $\pi$, the corresponding action-value function $q_\pi$ uniquely solves the Bellman equation \citep{bellman1957dynamic} for action values: 
\begin{equation}
    \label{eq:bellman_q}
    q_\pi(s,a) = \sum_{s' \in \mathcal{S}} \sum_{r \in \mathcal{R}} p(s',r|s,a) \left( r + \gamma \sum_{a' \in \mathcal{A}} \pi(a'|s') \, q_\pi(s',a') \right)
    \,.
\end{equation}

In the control setting, the agent's goal is to find an \emph{optimal} policy, $\pi_*$:
one such that $q_{\pi_*}(s,a) \geq q_\pi(s,a)$ for every policy $\pi$ and every state-action pair $(s,a)$.
Every optimal policy has the same action-value function, $q_*$, which uniquely solves the action-value Bellman optimality equation:
\begin{equation*}
    q_*(s,a) = \sum_{s' \in \mathcal{S}} \sum_{r \in \mathcal{R}} p(s',r|s,a) \left( r + \gamma \max_{a' \in \mathcal{A}} q_*(s',a') \right)
    \,.
\end{equation*}

Temporal-difference (TD) methods for control estimate action-value functions, either $q_\pi$ or $q_*$, from sample-based interaction with the environment.
Given a transition $(S_t,A_t,R_{t+1},S_{t+1})$, the agent then conducts an incremental update to its estimate of the action-value function, $Q(S_t,A_t)$.
In the \emph{off-policy} setting, the agent is assumed to select actions with probability $b(A_t|S_t)$, where $b$ is a behavior policy which differs from the target policy, $\pi$.

Most of the algorithms considered in this paper are off-policy methods.
They learn an action-value function $Q(s,a)$ that estimates either $q_\pi(s,a)$ or $q_*(s,a)$ using samples obtained from executing policy $b$ in the environment.
For example, Expected Sarsa \citep{john1994when} is an off-policy TD method that uses explicit knowledge of the target policy $\pi$ to approximate the solution to \Cref{eq:bellman_q} from sampled data:
\begin{equation}
    \label{eq:expected_sarsa}
    Q(S_t,A_t) \gets Q(S_t,A_t) + \alpha \left(R_{t+1} + \gamma \sum_{a' \in \A} \pi(a'|S_{t+1}) Q(S_{t+1},a') - Q(S_t,A_t) \right)
    \,,
\end{equation}
where $\alpha \in (0,1]$ is the step size of the update.
We revisit Expected Sarsa as an important baseline several times in this paper because it generalizes a number of fundamental bootstrapping algorithms \citep{van2011insights}.
For instance, when the target policy is greedy with respect to $Q$, the expectation in \Cref{eq:expected_sarsa} becomes equivalent to $\max_{a' \in \A} Q(s',a')$, yielding the Q-learning algorithm:
\begin{equation}
    \label{eq:q-learning}
    Q(S_t,A_t) \gets Q(S_t,A_t) + \alpha \left(R_{t+1} + \gamma \max_{a' \in \A}  Q(S_{t+1},a') - Q(S_t,A_t) \right)
    \,,
\end{equation}
which will be useful for the control case in \Cref{subsec:qvmax}.

\subsection{QV-learning}

We now begin to discuss methods that learn state values as a means to estimate action values.
The state-value function, $v_\pi$, is defined analogously to $q_\pi$ as
\begin{equation*}
    v_\pi(s) = \mathbb{E}_\pi[G_t \mid S_t = s]
\end{equation*}
and uniquely solves the Bellman equation
\begin{equation}
    \label{eq:bellman_v}
    v_\pi(s) = \sum_{a \in \A} \pi(a|s) \sum_{s' \in \S} \sum_{r \in \mathcal{R}} p(s',r|s,a) \, \big(r + \gamma \mathop{v_\pi(s')}\big)
    \,.
\end{equation}
The state-value function is related to the action-value function by $v_\pi(s) = \sum_{a \in \A} \pi(a|s) q_\pi(s,a)$.
This implies that an approximation of $v_\pi(s')$ can be substituted in \Cref{eq:bellman_q} to avoid the summation over actions.
QV-learning \citep{wiering2005qv} is an on-policy TD algorithm based on this concept.
The idea is to use TD(0) to independently learn $V(s) \approx v_b(s)$, while concurrently bootstrapping from these estimated state values to learn $Q(s,a) \approx q_b(s,a)$.
The specific value-function updates for QV-learning become
\begin{align}
    \label{eq:qv-learning_q}
    \input{equations/qv-learning_q}
    \,, \\*
    \label{eq:qv-learning_v}
    V(S_t) &\gets V(S_t) + \alpha \Big(R_{t+1} + \gamma V(S_{t+1}) - V(S_t) \Big)
    \,.
\end{align}
\citet[][Sec.~2]{wiering2005qv} is explicit that $Q$ is always updated before $V$ in each iteration.
Additionally, it is commonly assumed that both updates share the same step size, $\alpha$, as is shown above.
We analyze this algorithm in \Cref{subsec:on-policy_prediction}.

The main variant of QV-learning for off-policy control is QVMAX \citep{wiering2009qv}, which essentially substitutes the max operator from Q-learning into \Cref{eq:qv-learning_v} in an attempt to induce convergence to $q_*$ instead of $q_b$:
\begin{align}
    \tag{\ref*{eq:qv-learning_q}}
    \input{equations/qv-learning_q}
    \,, \\*
    \label{eq:qvmax_v}
    \input{equations/qvmax_v}
    \,.
\end{align}
The update to $Q(S_t,A_t)$ remains the same as \Cref{eq:qv-learning_q} and is still conducted prior to \Cref{eq:qvmax_v} at each time step.
We analyze this algorithm in \Cref{subsec:qvmax}.

\subsection{Dueling Q-learning}
\label{subsec:bg_dueling}

The basic idea behind Dueling Q-learning \citep{wang2016dueling} is to decompose the action-value function, $Q(s,a)$, into a state-value function, $V(s)$, and an advantage function, $\Adv(s,a)$.
The particular decomposition used by \citet{wang2016dueling} is
\begin{equation}
    \label{eq:dueling_mean}
    Q(s,a) = V(s) + \Adv(s,a) - \frac{1}{\lvert \mathcal{A} \rvert} \sum_{a' \in \A} \Adv(s,a')
    \,,
\end{equation}
where the last term is subtracted to make the solution unique (we further discuss this point in \Cref{subsec:soft_rdq}).
The updates to $V(s)$ and $\Adv(s,a)$ are then derived implicitly from Q-learning by reframing its update as stochastic semi-gradient descent.
Let $\vtheta$ be the vector of all learnable value-function parameters involved in \Cref{eq:dueling_mean}.
We therefore write $Q(s,a;\vtheta)$ for the action-value estimate of state-action pair $(s,a)$.
Additionally, let $\vtheta^-$ be the target parameters for bootstrapping \citep{mnih2015human}.
The Q-learning update in \Cref{eq:q-learning} is rewritten as the following squared-error loss minimization:
\begin{align}
    \nonumber
    \vtheta &\gets \vtheta - \alpha \mathop{\frac{1}{2}} \gradvtheta \left(R_{t+1} + \gamma \max_{a' \in \A} Q(S_{t+1},a';\vtheta^-) - Q(S_t,A_t;\vtheta) \right)^2 \\*
    \label{eq:q-learning_sgd}
    &= \vtheta + \alpha \left(R_{t+1} + \gamma \max_{a' \in \A} Q(S_{t+1},a';\vtheta^-) - Q(S_t,A_t;\vtheta)\right) \gradvtheta Q(S_t,A_t;\vtheta)
    \,.
\end{align}
Applying the chain rule to $\gradvtheta Q(S_t,A_t;\vtheta)$ in \Cref{eq:q-learning_sgd} according to the value decomposition in \Cref{eq:dueling_mean} yields the final dueling update.
In practice, the decomposition in \Cref{eq:dueling_mean} is represented by a branching and merging neural network, and the chain rule is handled by automatic differentiation.
In \Cref{subsec:dueling_q-learning}, we instead apply the chain rule to \Cref{eq:dueling_mean} manually, assuming tabular value functions, to derive a Dueling Q-learning update without function approximation.
This allows us to gain further insight into this fundamental algorithm beyond the deep RL setting.

\section{QV-learning Algorithms}
\label{sec:qv-learning}

In this section, we investigate QV-learning algorithms, which jointly learn an action-value function $Q(s,a)$ and a state-value function $V(s)$.
Although the fundamental QV-learning algorithm has existed for around 20 years and several studies have pointed to its effectiveness \citep[e.g.,][]{wiering2005qv,wiering2007two,sabatelli2020deep,modayil2023towards}, empirical and theoretical analysis has still been limited.

\subsection{On-Policy Prediction}
\label{subsec:on-policy_prediction}

In this subsection, we conduct a focused experiment to determine when QV-learning improves performance versus a single action-value function.
We then conclude with an expected convergence analysis of QV-learning.

Throughout this paper, we will revisit Expected Sarsa as a prototypical baseline for TD learning of action values.
Recall that Expected Sarsa's update rule is
\begin{equation}
    \tag{\ref*{eq:expected_sarsa}}  
    Q(S_t,A_t) \gets Q(S_t,A_t) + \alpha \Big( R_{t+1} + \gamma \underbrace{\sum_{a' \in \A} \pi(a'|S_{t+1}) Q(S_{t+1},a')}_{\approx v_\pi(S_{t+1})} - \mathop{Q(S_t,A_t)} \Big)
    \,.
\end{equation}
The comparison to Expected Sarsa is natural because QV-learning can be seen as learning an approximation, $V(S_{t+1})$, of the expected next action value, which we have underscored above.
This provides an indication as to when QV-learning may have an advantage over Expected Sarsa;
as the action-space cardinality $\abs{\A}$ of the MDP grows, Expected Sarsa requires many more sample interactions to estimate all of the action values in the expectation, whereas QV-learning can simply learn $V(S_{t+1})$ with a small number of samples that is roughly insensitive to the number of actions, $\abs{\A}$.

To test this, we design a parametric MDP experiment where the environment has $4$ states and $\abs{\A}$ actions.
Taking any action in any state triggers a transition to a random state (possibly the same one) with equal probability.
The agents' behavior policy is uniform random in every state.
Agents receive a reward of $+1$ whenever the first action, $a_0$, is taken, and a reward of $0$ otherwise.

Because the MDP dynamics are state-invariant, $v_b(s)$ is constant for all $s \in \mathcal{S}$.
Solving the state-value Bellman equation, \Cref{eq:bellman_v}, yields $v_b(s) = 1 \mathbin{/} \big((1-\gamma) \, \abs{\A}\big)$.
It follows that $q_b(s,a_i) = \indicator{i=0} + \gamma \mathbin{/} \big((1-\gamma) \, \abs{\A}\big)$ for all $a_i \in \mathcal{A}$, where $\mathbf{1}$ is the indicator function.
We set $\gamma = 0.99$.

We evaluate the agents' performance by measuring the root mean square (RMS) prediction error, $\|q_b - Q\|_2$, as a function of the number of environment interactions.
We normalize the errors by expressing them as a percentage of the initial RMS error, $\|q_b - Q_0\|_2$, where $Q_0$ is initialized with zeros.
We compare Expected Sarsa and QV-learning, training both agents for 20,000 time steps.
In \Cref{fig:on-policy_prediction} (left), we plot the learning curves for the case where $\abs{\A}=18$, the largest instantiation of our action set that we consider here.
We then average the results over 100 independent trials, with shading indicating 95\% confidence intervals.
The agents' step sizes are chosen from a natural-logarithmic grid search of 61 values over the interval $(0,1]$, similarly to the random walk experiment by \citet[][Sec.~12.1]{sutton2018reinforcement}.
The selection criterion is to minimize the area under the curve (AUC), which we normalized as a percentage.
Specifically, an AUC of 100\% indicates that the initial error did not change at all, whereas a lower percentage indicates faster average progress towards the fixed point.
In \Cref{fig:on-policy_prediction} (center), we plot the corresponding AUCs for each step size tested by the sweep.
The dashed horizontal line corresponds to the smallest AUC achieved by each agent.

\begin{figure*}[t]
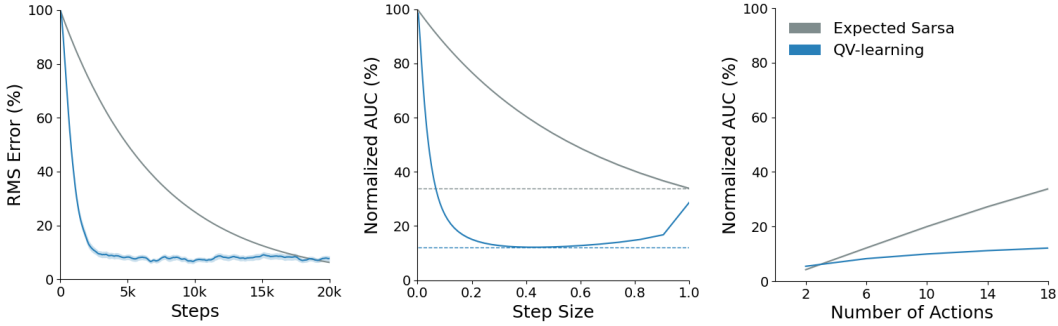

    \centering
    \includegraphics[width=0.32\textwidth]{on-policy_rms-vs-time_18actions.png}
    \hfill
    \includegraphics[width=0.32\linewidth]{on-policy_auc-vs-stepsize_18actions.png}
    \hfill
    \includegraphics[width=0.32\linewidth]{on-policy_auc-vs-actions.png}
    \caption{
        On-policy prediction performance of QV-learning compared to Expected Sarsa in the 4-state MDP.
        (Left)~Root mean square (RMS) prediction error versus the number of learning steps, after optimizing the step size, $\alpha \in (0, 1]$.
        (Center)~Area under the learning curve (AUC), normalized as a percentage, for each step size.
        The dashed lines correspond to the AUCs for the learning curves in the left subplot.
        (Right)~The smallest AUC obtained across all step sizes, as the number of MDP actions is increased.
        Averaged across 100 trials; shading represents 95\% confidence intervals.
        \looseness=-1
    }
    \label{fig:on-policy_prediction}
\end{figure*}

To investigate the scalability of these methods, we repeat this experiment setup for ${\abs{\A} \in \{2, 6, 10, 14, 18\}}$.
We then plot the best AUC achieved in each problem instance in \Cref{fig:on-policy_prediction} (right).
The slope of the resulting lines roughly correspond to the scalability of the algorithms.
We note that QV-learning's line is more horizontal, indicating better scaling to large action-space cardinalities, as we hypothesized earlier.
However, the trade-off for this improved sample efficiency is a much noisier update, which can be seen in \Cref{fig:on-policy_prediction} (left);
if we were to train for longer, Expected Sarsa would eventually obtain a more accurate solution.
This indicates that QV-learning is perhaps best suited for cases where learning a satisfactory---but possibly imperfect---value function is desired within a small number of samples.

These results do not, however, suggest that QV-learning fails to converge to the correct pair of on-policy value functions, $q_b$ and $v_b$.
The noisy behavior observed in \Cref{fig:on-policy_prediction} (left) is primarily due to the fact that the step size is not annealed.
The following theorem states that, in expectation, the joint update of QV-learning is a contraction mapping when we represent the two value functions as a concatenated vector of size $\abs{\S} + \abs{\S \times \A}$.

\begin{restatable}[QV-learning Contraction]{theorem}{qvlearningcontraction}
    \label{theorem:qv-learning_contraction}
    The expected QV-learning update corresponds to an affine joint operator
    $H \colon \qv
    \mapsto \vb + \mA \qv$,
    where
    $\vb = \begin{bsmallmatrix}
        \vr \\
        \mE_b \vr \\
    \end{bsmallmatrix}$
    and
    $\mA = \gamma \begin{bsmallmatrix}
        \vzero & \mP \\
        \vzero & \mE_b \mP \\
    \end{bsmallmatrix}$.
    The operator $H$ is a contraction mapping with its unique fixed point equal to $\qvb$.
\end{restatable}

\begin{proof}
    See \Cref{subapp:qv-learning_contraction}.
\end{proof}

The significance of this result is that it shows that both of QV-learning's value functions make progress (on average) towards their respective fixed points, without requiring $V$ to converge first.
This is in contrast to less formal convergence arguments for QV-learning, which may rely on the assumption that TD(0) would first converge to $v_b$ based on well-established convergence results, and then the update rule in \Cref{eq:qv-learning_q} would be able to bootstrap from it to extract $q_b$.
The issue with this two-timescale approach is that $V$ may never exactly equal $v_b$ after any finite amount of time, meaning that $Q$ would still incur some bootstrapping bias---a caveat that is directly addressed by considering the joint operator space, as we have done here.

The fact that $H$ is a contraction mapping implies that both $Q$ and $V$ converge to their respective fixed points even when updates are asynchronous, under the additional (but standard) technical assumptions that the step size $\alpha$ is appropriately annealed and the conditional variances of the updates are bounded \citep[see][Prop.~4.4]{bertsekas1996neuro}.
We note that the latter assumption is automatically satisfied by our MDP definition, which has finite sets of states, actions, and rewards.
Finally, we remark that our analysis considers only the prediction setting in which the behavior policy $b$ is fixed;
for the control case, we would likely need to invoke the ``greedy in the limit with infinite exploration'' (GLIE) assumption \citep{singh2000convergence} to prove eventual convergence to an optimal policy, but we leave this as an open problem.
\looseness=-1

\subsection{Off-Policy Control}
\label{subsec:qvmax}

The previous experiment shows that bootstrapping from state values when learning action values can be much more efficient than directly learning the action values.
Unfortunately, we cannot leverage this same technique for off-policy control, as this would require a model-free method for estimating $v_*$ directly from environment samples.

To circumvent this, \citet{wiering2009qv} introduced an off-policy variant of QV-learning called QVMAX which borrows the max operator from Q-learning when updating $V(S_t)$:
\begin{align}
    \tag{\ref*{eq:qv-learning_q}}
    \input{equations/qv-learning_q}
    \,, \\
    \tag{\ref*{eq:qvmax_v}}
    \input{equations/qvmax_v}
    \,.
\end{align}
\begin{wrapfigure}{r}{0.45\textwidth}
    \vspace{-0.1in}
    \centering
    \input{figures/qv_flowchart.tex}
    \caption{Depiction of bootstrapping in QV-learning variants.
    The arrows point from the bootstrapped value function to the value function being updated.}
    \label{fig:qv_flowchart}
\end{wrapfigure}
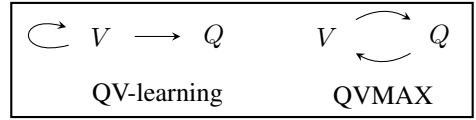
One consequence of this change is that now there is reciprocal bootstrapping between $Q$ and~$V$.
Previously, with QV-learning, we had only a unidirectional information flow:
$V$ bootstrapped from itself, and $Q$ bootstrapped from $V$.
Hence, $Q$ could not corrupt the state values in any way.
We illustrate this distinction in \Cref{fig:qv_flowchart}.
Although reciprocal bootstrapping is not necessarily an undesirable property, it is worth noting that a fundamental characteristic of the algorithm has changed.

A more serious consequence of this change is that the update now suffers from off-policy bias.
\citet{wiering2009qv} hypothesized that the max operator in \Cref{eq:qvmax_v} would induce convergence to $q_*$, as it does in Q-learning.
Unfortunately, this is not true, as we show in the next proposition.

\medskip  
\begin{restatable}{proposition}{qvmaxfp}
    \label{prop:qvmax_fp}
    Unless the behavior policy, $b$, satisfies
    ${\mathbb{E}_b[R_{t+1} + \gamma \mathop{v_*(S_{t+1})} \mid S_t = s]} = {\mathbb{E}_{\pi_*}[R_{t+1} + \gamma \mathop{v_*(S_{t+1})} \mid S_t = s]}$
    for all $s \in \S$, then $\qvstar$ is not a fixed point of QVMAX.
\end{restatable}

\begin{proof}
    See \Cref{subapp:prop_qvmax_fp}.
\end{proof}

The update fails to correct the distribution of the observed reward, $R_{t+1}$, which is collected by the behavior policy.
This is because \Cref{eq:qvmax_v} is \emph{not} conditioned on the state-action pair, $(S_t,A_t)$, as is normally the case for Q-learning, but is instead conditioned only on the state, $S_t$.
The freedom in the choice of $A_t$ means the expected reward observed from state $S_t$ now depends on behavior policy $b$.
\looseness=-1

To eliminate the bias, the update cannot depend explicitly on the sampled reward, $R_{t+1}$.
We therefore propose to modify \Cref{eq:qvmax_v} in the QVMAX updates such that we have
\begin{align}
    \tag{\ref*{eq:qv-learning_q}}
    \input{equations/qv-learning_q}
    \,, \\
    \label{eq:bc-qvmax_v}
    V(S_t) &\gets V(S_t) + \alpha \Big( \max_{a \in \A} Q(S_t,a) - V(S_t) \Big)
    \,.
\end{align}
We call this \emph{Bias-Corrected} QVMAX (BC-QVMAX).
This is the proper algorithm for learning $q_*$, as the update now approximates the Bellman optimality relationship $v_*(s) = \max_{a \in \A} q_*(s,a)$.

\medskip  
\begin{restatable}{proposition}{bcqvmaxfp}
    \label{prop:bc-qvmax_fp}
    $\qvstar$ is the unique fixed point of BC-QVMAX.
\end{restatable}

\begin{proof}
    See \Cref{subapp:prop_bc-qvmax_fp}.
\end{proof}

A peculiarity of the updates in \Cref{eq:qv-learning_q,eq:bc-qvmax_v} is that they closely resemble Q-learning in \Cref{eq:q-learning}.
In fact, if we set $\alpha = 1$, then \Cref{eq:bc-qvmax_v} simply becomes a table overwrite:
$V(S_t) \gets \max_{a \in \A} Q(S_t,a)$.
This would make the algorithm almost the same as Q-learning, but with a delay introduced before bootstrapping.
When $\alpha < 1$, then \Cref{eq:bc-qvmax_v} has an extra smoothing effect, mixing together stale estimates of the maximum action value in each state and exacerbating this delay.

We slightly modify the prediction experiment from \Cref{subsec:on-policy_prediction} to demonstrate the bias of QVMAX.
This time, we target a greedy policy with respect to the current action-value function, making Expected Sarsa equivalent to Q-learning.
We therefore compare Q-learning, QVMAX, and BC-QVMAX.
The optimal policy is to select $a_0$ unconditionally, making $q_*(s,a_0) = 1 \mathbin{/} (1-\gamma)$.
It follows that $q_*(s,a_i) = \gamma \mathbin{/} (1-\gamma)$ for $i \neq 0$.

All experiment and plotting procedures remain the same as before, except that now the RMS error is defined in terms of $q_*$:
i.e., $\|q_* - Q\|_2$.
We plot the results in \Cref{fig:qvmax}.
As can be seen in \Cref{fig:qvmax} (left), and as our theory predicted, QVMAX asymptotes substantially above the zero-error mark due to its biased update.
In contrast, BC-QVMAX achieves a much lower error.
However, as we also discussed above, the smoothing effect of BC-QVMAX makes it similar to, but rather slower than, Q-learning.
The results in \Cref{fig:qvmax} (left) are identical only because $\alpha=1$ happens to be the best step size for both methods, which makes them nearly equivalent (except for the delay mentioned previously).
However, for all $\alpha < 1$, we can see from \Cref{fig:qvmax} (center) that BC-QVMAX slightly lags behind Q-learning.

These results unfortunately indicate that off-policy control with QV-learning algorithms is much harder than on-policy prediction, since model-free estimation of $v_*$ is not straightforward.
While we note that this one experiment is not enough to rule out the viability of QVMAX approaches in this setting, it does provide convincing evidence against it.
In particular, the theoretically correct version of the algorithm, BC-QVMAX, is similar to a learned approximation of Q-learning which introduces delay into the bootstrapping.
This does not serve an immediately obvious benefit and appears to negatively impact sample efficiency, though there is the possibility that this smoothing effect may have utility in some settings (e.g., to enhance stability in tracking problems).

\begin{figure*}[t]
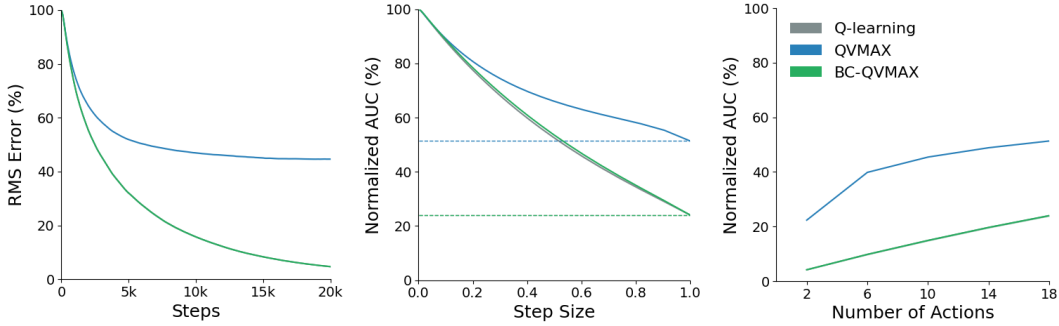

    \centering
    \includegraphics[width=0.32\textwidth]{qvmax_rms-vs-time_18actions.png}
    \hfill
    \includegraphics[width=0.32\linewidth]{qvmax_auc-vs-stepsize_18actions.png}
    \hfill
    \includegraphics[width=0.32\linewidth]{qvmax_auc-vs-actions.png}
    \caption{
        Off-policy control performance of BC-QVMAX compared to QVMAX and Q-learning in the 4-state MDP.
        Experiment setup is the same as that of \Cref{fig:on-policy_prediction}.
        (Note that Q-learning is sometimes eclipsed by BC-QVMAX.)
    }
    \label{fig:qvmax}
\end{figure*}

\section{AV-learning Algorithms}
\label{sec:av-learning}

Rather than learning explicit state- and action-value functions like QV-learning methods do, an alternative approach is to decompose the action-value function into constituent state-value and advantage functions that can be implicitly updated using gradient descent.
This \emph{dueling} strategy was introduced by \citet{wang2016dueling}, but we generalize it in this section.

\subsection{Dueling Q-learning}
\label{subsec:dueling_q-learning}

To obtain a tabular Dueling Q-learning update, we evaluate the chain rule in \Cref{eq:q-learning_sgd}, as discussed earlier in \Cref{subsec:bg_dueling}.
Q-learning's TD error is $\delta_t = R_{t+1} + \gamma \max_{a' \in \A} Q(S_{t+1},a') - Q(S_t,A_t)$.
Because $Q(s,a)$ is defined according to \Cref{eq:dueling_mean}, the chain rule preserves the error, $\delta_t$, inside the quadratic term and then scales it in the following manner:
\begin{align*}
    \Adv(S_t,a) &\gets \Adv(S_t,a) + \alpha \left(\indicator{a = A_t} - \frac{1}{\lvert\mathcal{A}\rvert}\right) \delta_t
    \,,
    \quad \forall a \in \A
    \,,\\*
    V(S_t) &\gets V(S_t) + \alpha \delta_t
    \,.
\end{align*}
It may seem unusual to write out these updates in this manner;
Dueling Q-learning was originally proposed for deep RL, where automatic differentiation would handle the partial derivatives.
Nevertheless, these updates are well-defined for tabular value functions.
Furthermore, they reveal a property behind this particular style of dueling.
At each time step, a single advantage value is incremented in proportion to $\delta_t$, and then the $\abs{\A}$ advantage values are decremented in the same proportion to $\delta_t \mathbin{/} \abs{\A}$.
This implies that the arithmetic mean of the advantages is an \emph{invariant} quantity;
if we initialize the advantage values to zero, then the mean will remain zero throughout training.

\subsection{Regularized Dueling Q-learning}
\label{subsec:soft_rdq}

\begin{wrapfigure}{r}{0.5\textwidth}
    \vspace{-0.18in}
    \centering
    \includegraphics[width=0.8\linewidth]{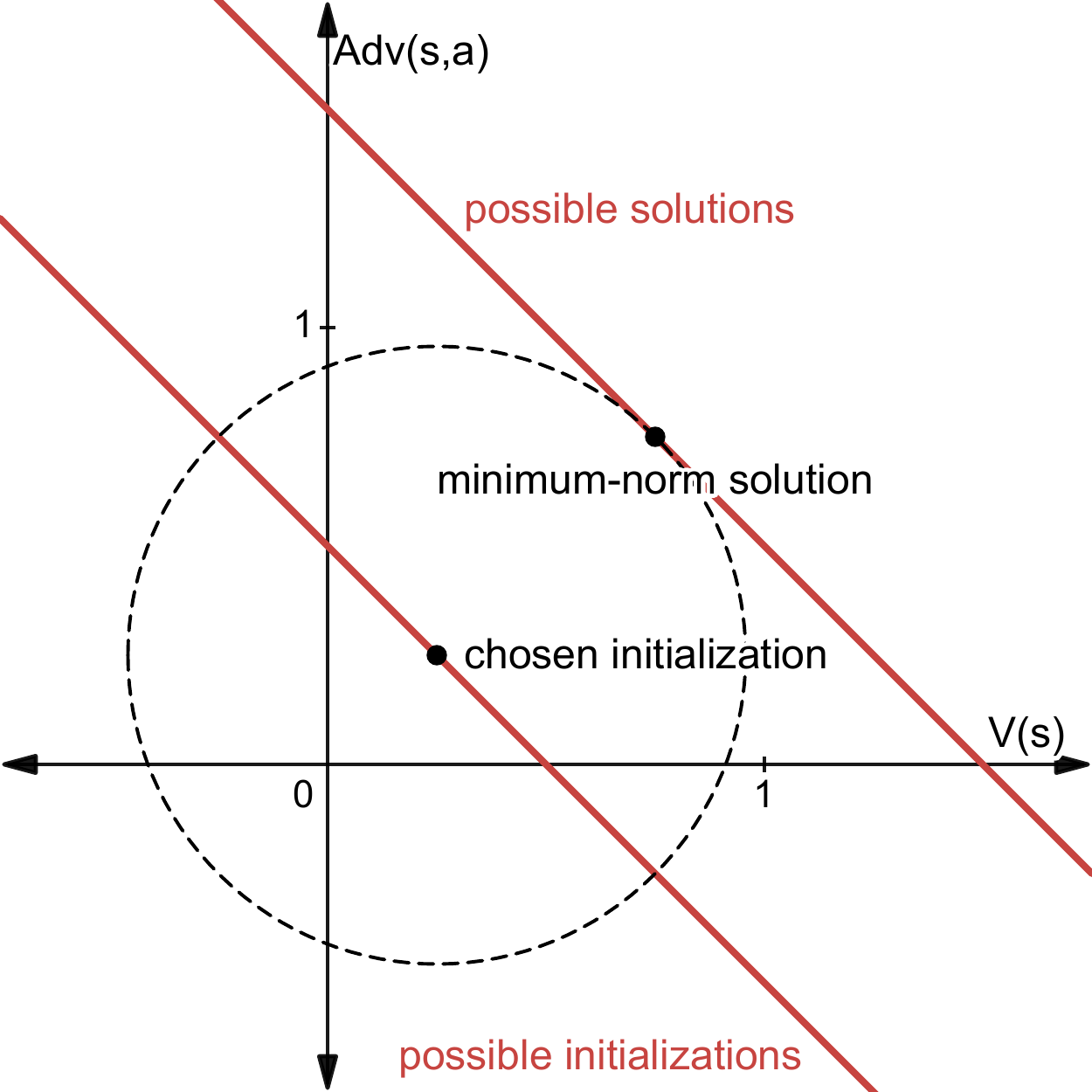}
    \caption{
        RDQ finds the minimum-norm solution to the underdetermined system of equations $Q(s,a) = V(s) + \Adv(s,a)$, $\forall a \in \A$.
        Thus, RDQ may require fewer updates than Dueling Q-learning to converge in practice.
    }
    \label{fig:rdq_geometry}
    \vspace{-0.1in}
\end{wrapfigure}

We propose a new dueling algorithm which does not rely on subtracting the identifiability term in \Cref{eq:dueling_mean}.
We start with the most general decomposition of the action-value function,
\begin{equation}
    \label{eq:av_decomposition}
    Q(s,a) = V(s) + \Adv(s,a)
    \,.
\end{equation}
This decomposition is unidentifiable \citep[][Sec.~3]{wang2016dueling};
we can add a constant to $V(s)$ and subtract the same constant from $\Adv(s,a)$ without changing $Q(s,a)$.
Consequently, there are infinitely many valid decompositions that satisfy \Cref{eq:av_decomposition} and no clear way for an algorithm to choose between them.

Dueling DQN seeks to accurately approximate both $v_\pi(s)$ and ${\Adv_\pi(s,a) \defeq q_\pi(s,a) - v_\pi(s)}$ in order to estimate $q_\pi(s,a)$, which motivates the subtraction of the mean advantage to get \Cref{eq:dueling_mean} from \Cref{eq:av_decomposition}.
The mean advantage is meant to coarsely approximate the expected advantage under the target policy, but sacrifices the precise semantics of $v_\pi$ and $\Adv_\pi$.
We propose an alternative approach in which we search for \emph{any} two functions, $V(s)$ and $\Adv(s,a)$, that accurately reconstruct $q_\pi(s,a)$, further relaxing these semantics in order to expedite learning.

Let us consider the underdetermined system of equations in \Cref{eq:av_decomposition} to motivate our algorithm.
For a particular state-action pair $(s,a)$, we can represent the specific solution found by Dueling Q-learning as an ordered pair:
$(V(s), \Adv(s,a))$.
This particular solution is just one of infinitely many ordered pairs that satisfy \Cref{eq:av_decomposition}, which form a negatively sloped line in the $V$-$\Adv$ plane:
$\Adv(s,a) = q_\pi(s,a) - V(s)$, depicted in \Cref{fig:rdq_geometry}.
When $Q(s,a)$ is initialized, the locus of valid initializations of $V(s)$ and $\Adv(s,a)$ is also a negatively sloped line:
$\Adv(s,a) = Q(s,a) - V(s)$, again depicted in \Cref{fig:rdq_geometry}.
These two lines are parallel, and therefore the shortest directional path between them is the vector orthogonal to both of them---the minimum-norm projection.
This fact remains true regardless of how $Q(s,a)$ is initialized, since the slope of both lines is always $-1$.

For this reason, we hypothesize that a good solution to \Cref{eq:av_decomposition} is the point on the solution line closest to the initialization of $V(s)$ and $\Adv(s,a)$, since it may require fewer updates to reach.
In practice, because the state-value and advantage functions are typically initialized near zero for both tabular and deep RL, we can approximate this solution more simply as the closest point to the origin.

We therefore propose to apply $l_2$ regularization to \Cref{eq:q-learning_sgd} to simultaneously restore identifiability to \Cref{eq:av_decomposition} and encourage the algorithm to find this minimum-norm approximation.
Specifically, we penalize the squared-error loss from \Cref{eq:q-learning_sgd} with a term of
\begin{equation}
    \label{eq:l2_penalty}
    \frac{1}{2} V(s)^2 + \frac{1}{2} \sum_{a \in \A} \Adv(s,a)^2
    \,.
\end{equation}
Since we now have $Q(s,a) = V(s) + \Adv(s,a)$ from \Cref{eq:av_decomposition} (no identifiability term), evaluating the chain rule in \Cref{eq:q-learning_sgd} gives us the following updates:
\begin{align}
    \label{eq:rdq_q}
    \Adv(S_t,a) &\gets (1-\beta) \Adv(S_t,a) + \alpha \indicator{a = A_t} \delta_t
    \,,
    \quad \forall a \in \A
    \,,\\
    \label{eq:rdq_v}
    V(S_t) &\gets (1-\beta) V(S_t) + \alpha \delta_t
    \,,
\end{align}
where $\beta \in [0,1)$ is the regularization coefficient.
We call this new algorithm Regularized Dueling Q-learning (RDQ).
Specifically, we refer to this variant with the $l_2$ penalty as \emph{Soft} RDQ.
In practice, we found Soft RDQ to be noisy in the tabular setting, since the state and advantage values become leaky---they get pushed towards zero by the $1-\beta$ factor, and then have to compensate using the stochastic TD error, $\delta_t$.
However, Soft RDQ is very useful in the case of function approximation, since the penalty in \Cref{eq:l2_penalty} is architecture-agnostic and can be easily combined with complex neural networks.
We experiment with this in \Cref{subsec:deep_rl}.

In the tabular setting, we do not need an $l_2$ penalty to obtain the desired minimum-norm solution.
This is because substituting $\Adv(s,a) = Q(s,a) - V(s)$ into the penalty in \Cref{eq:l2_penalty}, differentiating with respect to $V(s)$, and setting the derivative to zero yields
\begin{equation}
    \label{eq:invariant}
    V(s) = \sum_{a \in \A} \Adv(s,a)
    \,.
\end{equation}
Furthermore, when we remove the $l_2$ penalty from \Cref{eq:rdq_q,eq:rdq_v} by setting $\beta = 0$, we see that both $V(s)$ and exactly one advantage value, $\Adv(s,a)$, are incremented by exactly the same amount at each time step---indicating that \Cref{eq:invariant} is another invariant quantity.
This implies that, regardless of how $V$ and $\Adv$ are initialized, \Cref{eq:invariant} always remains true.\footnote{
    That is, up to a fixed translation determined by the difference at initialization:
    $V_0(s) - \!\! \sum\limits_{a \in \A} \Adv_0(s,a)$.
}
Therefore, rather than using an explicit $l_2$ penalty, we can directly apply the RDQ update without a penalty to achieve the minimum-norm solution in the tabular setting:
\begin{align*}
    \Adv(S_t,a) &\gets \Adv(S_t,a) + \alpha \mathop{\indicator{a = A_t}} \delta_t
    \,,
    \quad \forall a \in \A
    \,,\\*
    V(S_t) &\gets V(S_t) + \alpha \delta_t
    \,.
\end{align*}
We call this variant \emph{Hard} RDQ, since it explicitly follows the minimum-norm path (on average) without the use of a soft penalty.
Unfortunately, there does not appear to be a simple extension of this idea to the function approximation case.
If we were to apply these update rules under function approximation, the different gradients calculated for $V(s)$ and $\Adv(s,a)$ would violate the invariant quantity in \Cref{eq:invariant} and lose the minimum-norm convergence property.

To test this new algorithm, we repeat the off-policy control experiment from \Cref{subsec:qvmax}.
We once again use Q-learning as a baseline, and then additionally compare Dueling Q-learning with Hard RDQ.
The only minor modifications we make is that we now change the suboptimal reward from $0$ to $-1$ (which does not change the optimal policy), set $\gamma = 0.999$, and initialize $V(s)$ and $\Adv(s,a)$ using zero-mean Gaussian noise with a standard deviation of $2$.
The rationale for the nonzero reward is so that $Q$ does not correctly estimate either of the actions' true values initially (on average).

We plot the results in \Cref{fig:dueling}.
All experiment and plotting procedures remain the same as before.
In contrast to the QVMAX algorithm tested earlier, both dueling methods dramatically outperform Q-learning, showing that the advantage decomposition is a highly effective strategy in general.
In the largest instantiation of our MDP, where $\abs{\A} = 18$, Hard RDQ slightly outperforms Dueling Q-learning across the range of step sizes (see \Cref{fig:dueling}; left, center).
The spike at $\alpha = 0.5$ is due to the fact that AV-learning methods sum the updates to $V$ and $\Adv$, making the total effective step size equal to $1$ and therefore unstable.
Across the various MDP instances (see \Cref{fig:dueling}; right), Hard RDQ performs slightly better than Dueling Q-learning in a small majority of cases, but both algorithms perform very well overall.

\begin{figure*}[t]
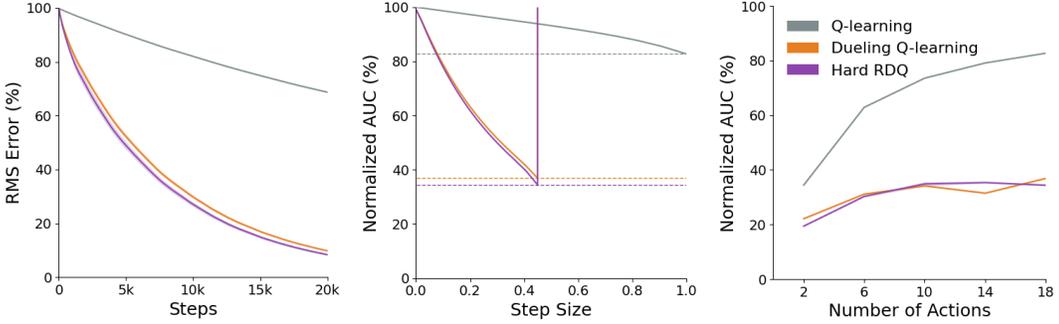

    \centering
    \includegraphics[width=0.32\textwidth]{dueling_rms-vs-time_18actions.png}
    \hfill
    \includegraphics[width=0.32\linewidth]{dueling_auc-vs-stepsize_18actions.png}
    \hfill
    \includegraphics[width=0.32\linewidth]{dueling_auc-vs-actions.png}
    \caption{
        Off-policy control performance of Hard RDQ compared to Dueling Q-learning and Q-learning in the 4-state MDP.
        Experiment setup is the same as that of \Cref{fig:on-policy_prediction}.
    }
    \label{fig:dueling}
\end{figure*}

\subsection{Deep RL Experiments}
\label{subsec:deep_rl}

Because Dueling Q-learning was originally proposed as a network architecture for Deep Q-Network \citep[DQN;][]{mnih2015human}, we test the performance of RDQ in a deep RL setting, where a neural network is used to approximate the action-value function.
We implement our agents using the PFRL library \citep{pfrl}.
Our benchmark is the MinAtar domain \citep{young2019minatar}, which includes five Atari-like games:
Asterix, Breakout, Freeway, Seaquest, Space Invaders.
The state representation for MinAtar is $10 \times 10$ multi-channel binary images displaying various objects and their velocities.
The reward function is the incremental game score.
\looseness=-1

We compare the soft variant of RDQ against DQN and Dueling DQN, using the hyperparameters used by \citet{ceron2021revisiting}.
The action-value function $Q(s,a;\vtheta_t)$ is approximated by a neural network, where $\vtheta_t$ is the network parameters at time $t$.
The network architecture we use for DQN is the same used by \citet{young2019minatar}:
a 16-filter, $3 \times 3$ convolutional layer, a 128-unit dense layer, and a final linear layer which maps to the $\lvert \A \rvert$ action values.
All layers except the last apply a rectified linear unit (ReLU) activation function.
We apply LeCun normal initialization \citep{lecun2002efficient} to the network parameters.

For RDQ and Dueling DQN, we create a dueling architecture following
\citeauthor{wang2016dueling}'s \citeyear{wang2016dueling} procedure.
We duplicate the 128-unit hidden layer to branch the network in parallel streams, and then separately map these into linear outputs of size $\lvert \A \rvert$ and $1$ for estimating $\Adv(s,a;\vtheta_t)$ and $V(s;\vtheta_t)$, respectively.
These are summed together (with broadcasting) to compute $Q(s,a;\vtheta_t)$.
Dueling DQN additionally subtracts the identifiability term, $\frac{1}{\lvert \A \rvert} \sum_{a' \in \A} \Adv(s,a';\vtheta_t)$, per \Cref{eq:dueling_mean}.

The agents execute an $\epsilon$-greedy policy, where $\epsilon$ is fixed to $1$ for the first 1,000 time steps of training and then linearly annealed to $0.01$ over the next 250,000 time steps.
Each transition, $(S_t,A_t,R_{t+1},S_{t+1})$, is stored in a replay buffer with a capacity of 100,000 transitions.
We let $\mathcal{D}_t$ denote the replay memory at time $t$.
Every $4$ time steps, the agents update the parameters using the Adam optimizer \citep{kingma2015adam} with a learning rate of $2.5 \times 10^{-4}$ and a denominator constant of $\epsilon = 3.125 \times 10^{-4}$.
Both DQN and Dueling DQN are trained to minimize the loss
\begin{equation}
    \label{eq:dqn_loss}
    \L^\text{DQN}_t \defeq \frac{1}{2} \,\, \E \! \left[ \bigg(R + \gamma \max Q(S,A;\vtheta^-_t) - Q(S,A;\vtheta_t)\bigg)^{\!2} \right]
    \,,
\end{equation}
where $\vtheta^-_t$ is the target network parameters copied from $\vtheta_t$ every 1,000 time steps.
The expectation in \Cref{eq:dqn_loss} is taken over the uniform distribution of samples $(S,A,R,S^\prime)$ in $\mathcal{D}_t$;
in practice, we approximate this with minibatches of size $32$.
RDQ inherits the same loss, but adds the regularization term:
\looseness=-1
\begin{equation*}
    \L^\text{RDQ}_t \defeq \L^\text{DQN}_t + \frac{\beta}{2} \,\, \E \! \left[ V(S;\vtheta_t)^2 + \sum_{a \in \mathcal{A}} \Adv(S,a;\vtheta_t)^2 \right]
    \,.
\end{equation*}
We choose $\beta = 10^{-3}$ for the regularization strength;
we did not tune this value.

Each agent was trained for a total of 10 million time steps.
Every 1 million time steps, the agent was evaluated with an $\epsilon$-greedy policy for 1,000 episodes with $\epsilon=0.01$.
In \Cref{fig:minatar}, we plot the mean undiscounted return for the evaluation episodes as a function of training time;
these results are averaged over 30 independent trials and the shading indicates 95\% confidence intervals.

\Cref{fig:minatar} depicts our results. 
We see that between DQN and Dueling DQN, there does not appear to be a clearly superior algorithm, and one algorithm may outperform the other depending on the environment.
However, we also see that RDQ significantly outperforms DQN and Dueling DQN in all five environments.
Note that RDQ shares the same architecture as Dueling DQN.
Although $\beta = 10^{-3}$ works well in this setting, more experiments would be needed to determine its efficacy in other domains.
Given that we did not tune $\beta$, it is also possible that performance could improve if tuned.
\looseness=-1

\begin{figure}[t]
    \centering
    \includegraphics[width=0.32\columnwidth]{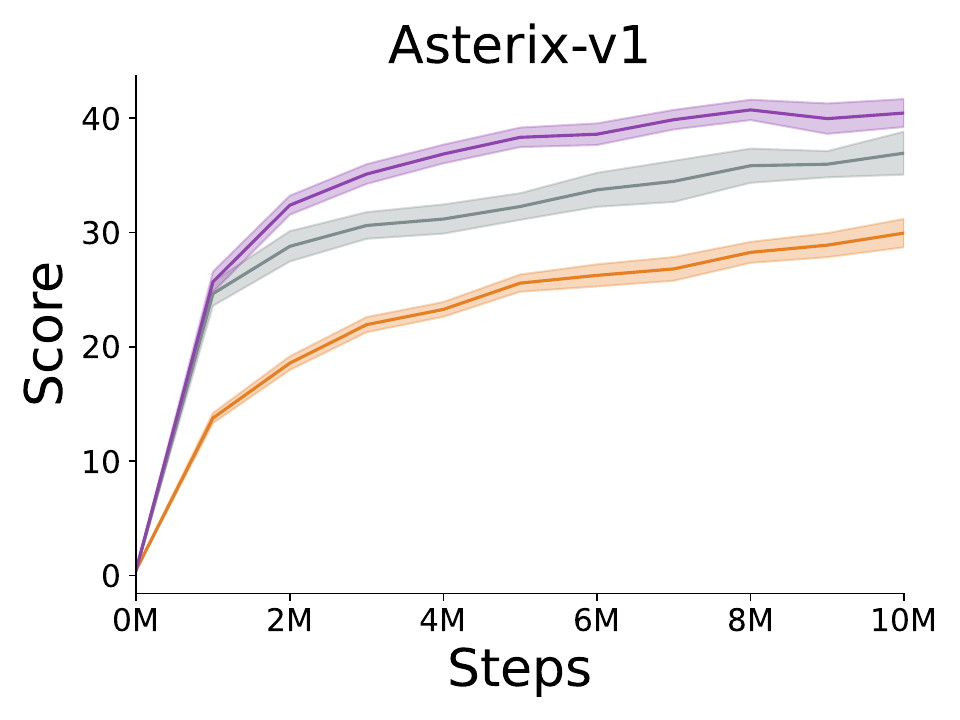}
    \hfill
    \includegraphics[width=0.32\columnwidth]{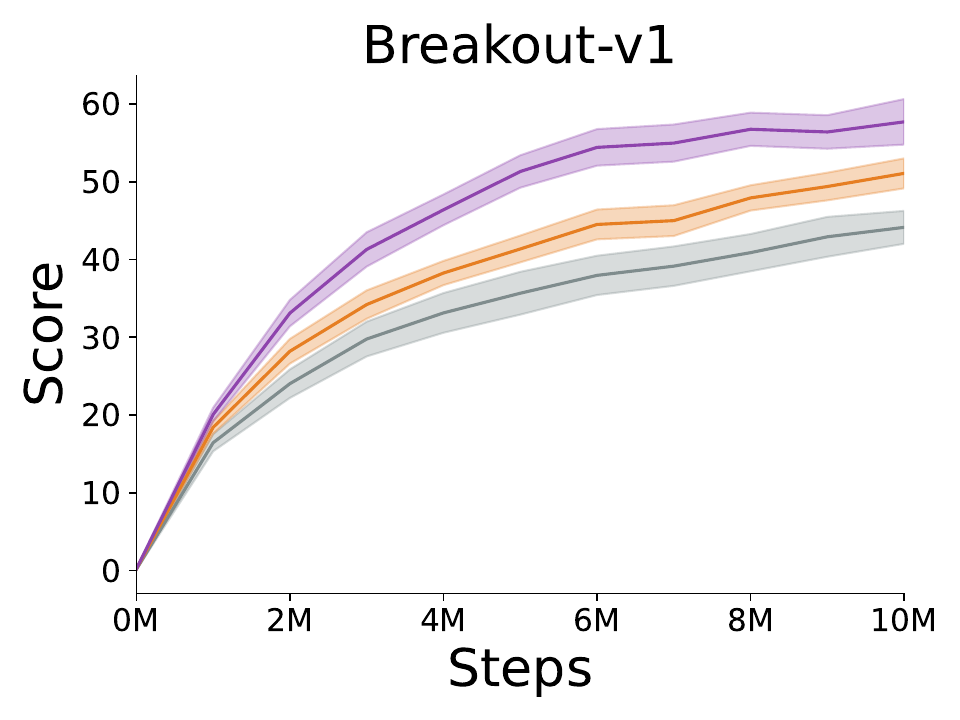}
    \hfill
    \includegraphics[width=0.32\columnwidth]{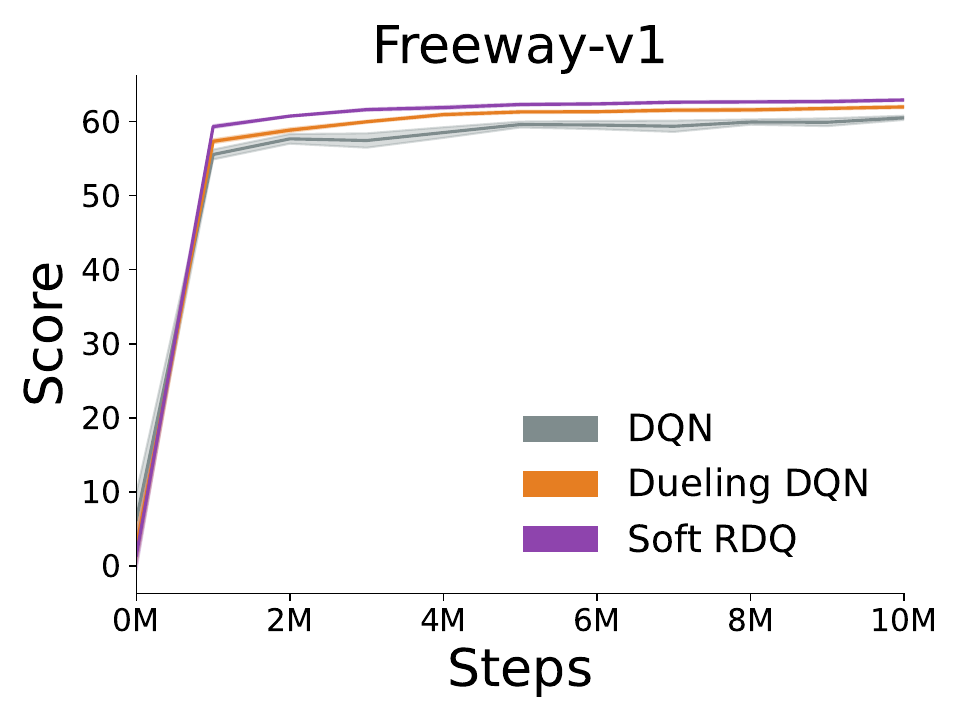}
    \\
    \hspace{\fill}
    \includegraphics[width=0.32\columnwidth]{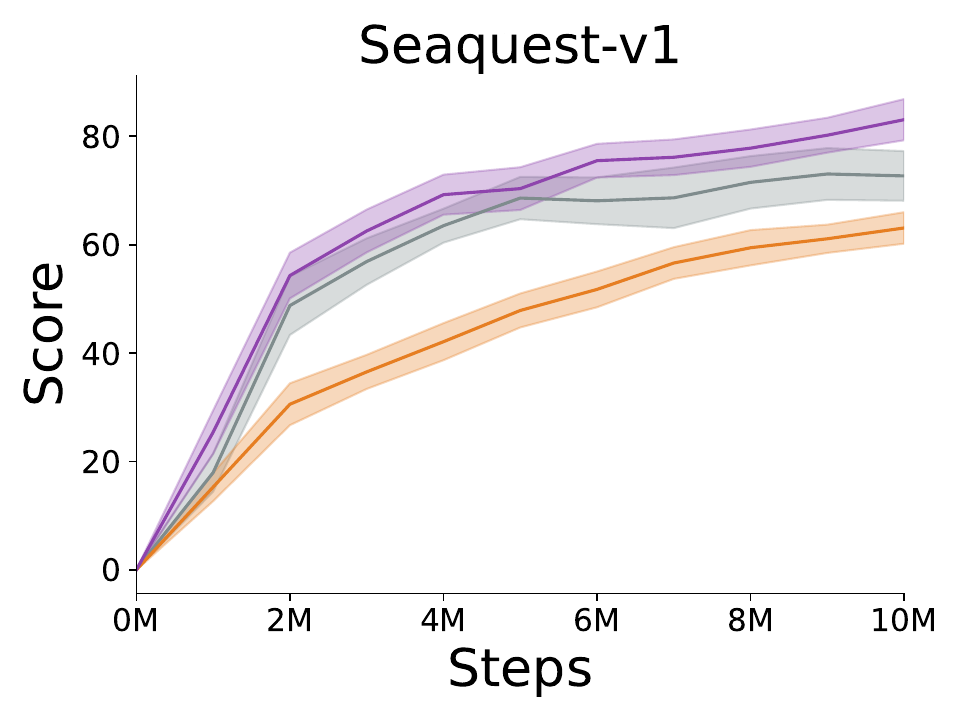}
    \hfill
    \includegraphics[width=0.32\columnwidth]{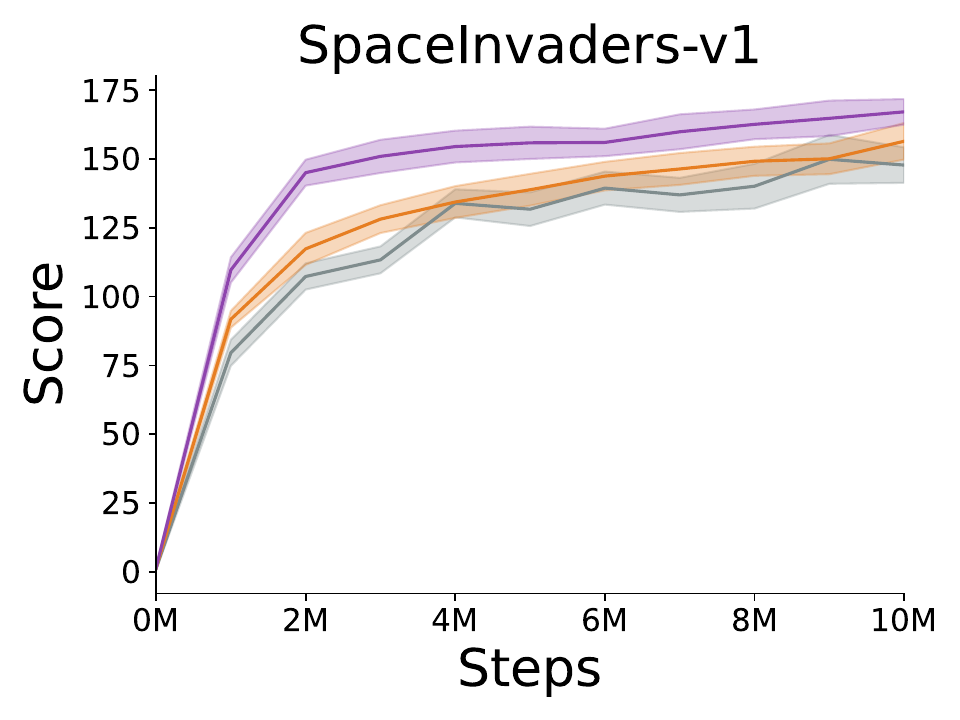}
    \hspace{\fill}
    \caption{Deep RL results for five MinAtar games.
    Averaged across 30 trials; shading represents 95\% confidence intervals.}
    \label{fig:minatar}
\end{figure}

\section{Conclusion}

We made several advances in the understanding of TD algorithms in which state-value functions are learned in tandem with action-value functions.
We showed that QV-learning algorithms have benefits for on-policy prediction---the original setting for which they were proposed---but tend to be much slower when the target policy is greedy.
We also demonstrated that the prevailing off-policy control variant of QV-learning, QVMAX, is biased, and we introduced a new variant called BC-QVMAX which empirically restores convergence.
Lastly, we introduced a novel Dueling Q-learning method called RDQ, which does not require the subtraction of an identifiability term, and has desirable properties in terms of the geometry of the learned value functions.
In a deep RL setting based on the MinAtar games, we showed that RDQ greatly outperforms Dueling DQN, despite having identical network architectures and the same, well-tuned hyperparameters.
Although there is still much to learn about these algorithms, our analysis helps to clarify their efficacy and distinguishing properties, and has already demonstrated potential for the development of new and effective RL algorithms.
\looseness=-1

%% file: equations/qv-learning_q
Q(S_t,A_t) &\gets Q(S_t,A_t) + \alpha \Big(R_{t+1} + \gamma V(S_{t+1}) - Q(S_t,A_t) \Big)

%% file: equations/qvmax_v
V(S_t) &\gets V(S_t) + \alpha \Big( R_{t+1} + \gamma \max_{a' \in \A}  Q(S_{t+1},a') - V(S_t) \Big)

%% file: figures/qv_flowchart.tex
\vbox to 1.5cm{
    \begin{tikzpicture}[node distance=1.5cm, >=stealth]
        \node[state, draw=none] (v1) {$V$};
        \node[state, draw=none, right of=v1] (q1) {$Q$};
        \draw (v1) edge[loop left] node{} (q1)
              (v1) edge[right] node[below=0.5cm]{QV-learning} (q1);
        \node[state, draw=none, right of=q1] (v2) {$V$};
        \node[state, draw=none, right of=v2] (q2) {$Q$};
        \draw (v2) edge[bend left] node{} (q2)
              (q2) edge[bend left] node[below=0.2cm]{QVMAX} (v2);
        \draw [black, thick] (current bounding box.south west) rectangle (current bounding box.north east);
    \end{tikzpicture}
}

%% file: paper_appendix.tex
\section{Proofs}
\label{app:proofs}

\begin{table}[b]
    \renewcommand{\arraystretch}{2}
    \centering
    \caption{Fundamental operators for MDPs \citep[][Ch.~2]{daley2025multistep}.}
    \begin{tabular}{ccl}
        \toprule
        Symbol & Input/Output & Definition \\
        \midrule
        $\mP$ & $\RS \to \RSA$ & $(P v)(s,a) \defeq \sum\limits_{s' \in \S} \mathop{v(s')} \sum\limits_{r \in \mathcal{R}} \mathop{p(s',r \mid s,a)}$ \\
        $\mE_\pi$ & $\RSA \to \RS$ & $(E_\pi q)(s) \defeq \sum\limits_{a \in \A} \mathop{\pi(a \mid s)} q(s,a)$ \\
        $E$ & $\RSA \to \RS$ & $(E q)(s) \defeq \max\limits_{a \in \A} q(s,a)$ \\
        \bottomrule
    \end{tabular}
    \label{tab:mdp_ops}
\end{table}

This section contains the proofs for all theoretical results in the paper.
These results rely on the analysis of dynamic-programming operators for MDPs, which we introduce here rather than \Cref{sec:background} for clarity of exposition.

In this context, value functions can be represented as vectors:
$\vv \in \RS$ for state values and $\vq \in \RSA$ for action values.
Each element of a vector corresponds to the value estimate for a state or a state-action pair;
the order of the elements does not matter as long as it is consistent.
Likewise, the expected reward
\begin{equation*}
    r(s,a) \defeq \sum_{r \in \mathcal{R}} \mathop{r} \sum_{s^\prime \in \S} \mathop{p(s',r \mid s,a)}
\end{equation*}
is an action-value function which can be represented as a vector, $\vr \in \RSA$.

Operators are then mappings between these vector spaces, which often represent an expected TD update applied simultaneously to every element.
In other words, TD methods are asynchronous, stochastic approximations of their underlying dynamic-programming operators.
Analyzing these operators gives insight into convergence properties.

We adopt the operator convention of \citet{daley2025multistep}, summarized in \Cref{tab:mdp_ops}, which decomposes the standard Bellman operators into three fundamental operators:
one for state transitions ($\mP$) and two for action selection ($\mE_\pi$ and $E$).
The major benefit to this notation is that these fundamental operators are never overloaded;
each applies to either $\vv$ or $\vq$ but not both.
This allows us to expand the Bellman operators for $\vv$ and $\vq$ into unambiguous expressions below---especially useful for analyzing QV-learning methods.

Given these preliminaries, the Bellman operator, $T_\pi$, is overloaded such that
\begin{align*}
    T_\pi \vq &\defeq \vr + \gamma \mP \mE_\pi \vq
    \,,\\
    T_\pi \vv &\defeq \mE_\pi (\vr + \gamma \mP \vv)
    \,.
\end{align*}
The Bellman expectation equations, \Cref{eq:bellman_q,eq:bellman_v}, can now be expressed succinctly as $\qpi = T_\pi \qpi$ and $\vpi = T_\pi \vpi$.
In other words, the unique fixed point of $T_\pi$ is either $\qpi$ or $\vpi$ depending on whether it acts on action values or state values, respectively.
Additionally, these fixed points are related by $\vpi = \mE_\pi \qpi$.

The Bellman optimality operator, $T$, is overloaded similarly but uses the greedy-policy operator $\mE$ instead of the expected-policy operator $\mE_\pi$:
\begin{align*}
    T \vq &\defeq \vr + \gamma \mP E \vq
    \,,\\
    T \vv &\defeq E (\vr + \gamma \mP \vv)
    \,.
\end{align*}
The subtle difference here is that $T$ is a nonlinear operator and admits the optimal value functions as the fixed points:
$\qstar = T \qstar$ and $\vstar = T \vstar$.
These fixed points are analogously related by $\vstar = E \qstar$.

These complete our operator definitions.
Our following theoretical results in the remainder of this section are derived from the fundamental properties of these operators.

\subsection{Proof of \Cref{theorem:qv-learning_contraction}}
\label{subapp:qv-learning_contraction}

\qvlearningcontraction*

\begin{proof}
The operator updates corresponding to the QV-learning, \Cref{eq:qv-learning_q,eq:qv-learning_v}, are
\begin{align*}
    \vq &= \vr + \gamma \mP \vv
    \,,\\
    \vv &= \mE_b (\vr + \gamma \mP \vv)
    \,.
\end{align*}
The fixed points of these updates are individually $\vq_b$ and $\vv_b$, respectively, which we show now by substituting these quantities.
The first update returns $\vq_b$ because of the relation $\vv_b = \mE_b \vq_b$:
\begin{equation*}
    \vr + \gamma \mP \vv_b = \vr + \gamma \mP \mE_b \vq_b = T_b \vq_b = \vq_b
    \,.
\end{equation*}
The second update is just equivalent to the Bellman operator, $T_b$, applied to $\vv$, which admits $\vv_b$ as its unique fixed point.

We next write these updates as a joint operator update
\begin{equation*}
    \qvbig \gets \underbrace{\begin{bmatrix}
        \vr \\
        \mE_b \vr
    \end{bmatrix}}_{\vb}
     +\ \underbrace{\gamma \begin{bmatrix}
         \vzero & \mP \\
         \vzero & \mE_b \mP
     \end{bmatrix}}_{\mA}
     \qvbig
     \,,
\end{equation*}
which we just established has a fixed point at $\qvb$.
Moreover, this is an affine operator of the form $H \colon \vy \mapsto \vb + \mA \vy$, where we let $\vy = \qv$ for brevity.
To complete the proof, we show that $H$ is a maximum-norm contraction mapping, which implies the fixed point $\qvb$ is unique.
Because each row of $\mA$ is a probability distribution scaled by $\gamma$, it follows that $\maxnorm{\mA} = \gamma$.
For any two vectors $\vy, \vy^\prime \in \RSAS$, we therefore have
\begin{equation*}
    \maxnorm{H \vy - H \vy^\prime} = \maxnorm{\mA (\vy - \vy^\prime)} \leq \gamma \maxnorm{\vy - \vy^\prime}
    \,,
\end{equation*}
so $H$ is indeed a maximum-norm contraction mapping.
By the Banach fixed-point theorem \citep{banach1922operations}, the fixed point $\qvb$ is unique and hence $\lim\limits_{i \to \infty} H^i \vy = \qvb$ for any initial vector $\vy$.
\end{proof}

\subsection{Proof of \Cref{prop:qvmax_fp}}
\label{subapp:prop_qvmax_fp}

\qvmaxfp*

\begin{proof}
The operator updates corresponding to \Cref{eq:qv-learning_q,eq:qvmax_v} are
\begin{align*}
    \vq &\gets \vr + \gamma \mP \vv
    \,,\\
    \vv &\gets \mE_b (\vr + \gamma \mP E \vq)
    \,.
\end{align*}
If $\qvstar$ were a fixed point of QVMAX, then substituting these vectors into the above operator updates would return the same result.

The first update (which reassigns $\vq$) correctly remains invariant when applied to the fixed point.
This is because $\vv_* = E \vq_*$ and therefore
\begin{align*}
    \vr + \gamma \mP \vv_*
    &= \vr + \gamma \mP E \vq_* \\
    &= T \vq_* \\
    &= \vq_*
    \,.
\end{align*}
However, the second update (which reassigns $\vv$) does not remain invariant when applied to the fixed point---the root cause of the bias in QVMAX.
This can be seen because
\begin{align*}
    \mE_b (\vr + \gamma \mP E \vq_*)
    &= \mE_b (\vr + \gamma \mP \vv_*) \\
    &\neq E (\vr + \gamma \mP \vv_*) \\
    &= T \vv_* \\
    &= \vv_*
    \,.
\end{align*}
This mismatch stems from the fact that the updates to $\vv$ are conditioned only on the states of the MDP and do not correct for the influence of taken actions.
As such, the updates are subject to changes to the reward and next-state distributions induced by the particular behavior policy, $b$.
This manifests as the discrepancy $\mE_b \neq E$ in the above equations.
These two operators only coincide in the serendipitous case that $b$ happens to simultaneously maximize the expected $1$-step returns, $\vr + \gamma \mP \vv$, but not in general.
Ultimately, the above analysis demonstrates that $\qvstar$ cannot be the fixed point of QVMAX for an arbitrary behavior policy and MDP, which completes the proof.
\end{proof}

\subsection{Proof of \Cref{prop:bc-qvmax_fp}}
\label{subapp:prop_bc-qvmax_fp}

\bcqvmaxfp*

\begin{proof}
The operator updates corresponding to BC-QVMAX, \Cref{eq:qv-learning_q,eq:bc-qvmax_v}, are
\begin{align*}
    \vq &\gets \vr + \gamma \mP \vv
    \,,\\
    \vv &\gets E \vq
    \,.
\end{align*}
We can unroll these updates by substituting them into each other---exploiting the fact that the updates are interleaved and not conducted simultaneously.

The first update (which reassigns $\vq$) becomes
\begin{align*}
    \vq &\gets \vr + \gamma \mP \vv \\
    &= \vr + \gamma \mP E \vq \\
    &= T \vq
    \,.
\end{align*}
The second update (which reassigns $\vv$) becomes
\begin{align*}
    \vv &\gets E \vq \\
    &= E (\vr + \gamma \mP \vv) \\
    &= T \vv
    \,.
\end{align*}
Both unrolled updates are equivalent to their overloaded Bellman optimality operators, $T$, which are contraction mappings and respectively admit $\vq_*$ and $\vv_*$ as their unique fixed points.
Therefore, $\qvstar$ is the unique fixed point of BC-QVMAX.
\end{proof}